%% file: neurips_2025.tex
\documentclass{article}

\usepackage[preprint]{neurips_2026}

\usepackage{amsthm}
\newtheorem{proposition}{Proposition}

\usepackage{amsmath, amssymb, amsfonts}
\usepackage{wrapfig} 
\usepackage{booktabs}    
\usepackage{multirow}    
\usepackage{enumitem}
\usepackage[]{mdframed}
\usepackage{framed}
\usepackage{booktabs}
\usepackage{multirow}
\usepackage{dsfont}
\usepackage{makecell}
\usepackage{enumitem, enumerate}
\usepackage{tikz}
\usepackage{booktabs}      
\usepackage{xcolor}        
\usepackage{threeparttable}
\usetikzlibrary{arrows.meta}
\usepackage{amsmath, amssymb}
\usepackage[noend]{algpseudocode}
\usepackage{algorithm}
\usepackage{algorithmicx}
\usepackage{algpseudocode}
\usepackage[most]{tcolorbox}
\tcbuselibrary{breakable, listings}
\usepackage{listings}
\usepackage{colortbl}
\usepackage{float, caption} 

\usepackage{algorithm}
\usepackage{algpseudocode}

\usepackage[utf8]{inputenc} 
\usepackage[T1]{fontenc}    
\usepackage{hyperref}       
\usepackage{url}            
\usepackage{booktabs}       
\usepackage{amsfonts}       
\usepackage{nicefrac}       
\usepackage{microtype}      
\usepackage{xcolor}         
\definecolor{lightgray}{gray}{0.9}

\usepackage{microtype}
\usepackage{hyperref}
\usepackage[capitalize,noabbrev,nameinlink]{cleveref}
\usepackage{url}

\tcbuselibrary{listings,breakable}
\usepackage{graphicx}

\definecolor{ForestGreen}{RGB}{34,139,34}
\definecolor{Crimson}{RGB}{220,20,60}
\usepackage{lineno}
\usepackage{soul} 
\usepackage{xspace}
\usepackage{enumerate}
\newcommand{\piroll}{\pi_{\text{rollout}}}
\newcommand{\pitrain}{\pi_{\text{train}}}
\title{AIS: Adaptive Importance Sampling for Quantized RL}

\author{Jiajun Zhou$^{1,2}$\thanks{Equal contribution.}   \thanks{Work done during internship at Huawei.} , Wei Shao$^1$\footnotemark[1] ,  Lingchao Zheng$^1$ ,  Yuwei Fan$^1$, Ngai Wong$^{2}$  \\
$^1$ Huawei 
$^2$ The University of Hong Kong
}

\begin{document}

\maketitle

\begin{abstract}

Reinforcement learning (RL) for large language models (LLMs) is dominated by the cost of rollout generation, which has motivated the use of low-precision rollouts (e.g., FP8) paired with a BF16 trainer to improve throughput and reduce memory pressure. This introduces a \emph{rollout--training mismatch} that biases the policy gradient and can cause training to collapse outright on reasoning benchmarks. We show that the mismatch is non-stationary and acts as a double-edged sword: early in training it provides a stochastic exploration bonus, exposing the gradient to trajectories the trainer would otherwise under-sample, but the same perturbation transitions into a destabilizing source of bias as the policy concentrates.

To solve this, we propose \textbf{Adaptive Importance Sampling (AIS)}, a correction framework that adjusts the strength of its intervention on a per-batch basis. AIS combines three real-time diagnostics, \emph{weight reliability}, \emph{divergence severity}, and \emph{variance amplification}, into a single mixing coefficient that interpolates between the uncorrected and fully importance-weighted gradients, suppressing the destabilizing component of the mismatch while preserving its exploratory benefit. We integrate AIS into GRPO and evaluate it on the diffusion-based LLaDA-8B-Instruct and the autoregressive Qwen3-8B and Qwen3.5-9B across mathematical reasoning and planning benchmarks. AIS matches the BF16 baseline on most tasks while retaining the $1.5$ to $2.76\times$ rollout speedup of FP8.

\end{abstract}

\input{sections/intro}
\input{sections/method_new}
\input{sections/experiments}

\newpage
\bibliographystyle{plainnat}
\bibliography{citation}
\clearpage


\appendix
\input{sections/appendix}



\end{document}

%% file: sections/intro.tex
\section{Introduction}

Reinforcement learning (RL) has become a dominant post-training paradigm for large language models (LLMs), driving recent advances in reasoning~\citep{yue2025vapo}, code generation~\citep{deepcoder2025, liu2025code}, and mathematics~\citep{deepscaler2025, yu2025dapo} through optimization against verifiable reward signals~\citep{guo2025deepseek, team2025kimi}. A single training step interleaves three stages: trajectory rollout, log-probability evaluation, and a policy gradient update. Among these, the rollout stage dominates both wall-clock time and memory, accounting for up to $70\%$ of end-to-end latency~\citep{zheng2024sglang, he2025history, zheng2025act}, with the imbalance further amplified on reasoning tasks that demand long chain-of-thought traces.

To accelerate this bottleneck, the recent large-scale industry has increasingly adopted low-precision quantization at the rollout stage. DeepSeek V3.2 integrates FP8 attention into its post-training pipeline~\citep{deepseekv32}, and DeepSeek V4 pushes rollout precision further down to FP4~\citep{deepseekai2026deepseekv4}. Open-source frameworks such as veRL~\citep{sheng2024hybridflow} provide production-level FP8 rollout pipelines for GRPO training on Qwen3-class models~\citep{nvidia2026fp8rl, qiu2026fp8}. Across these reports, FP8 rollout consistently delivers $1.2$ to $1.5\times$ speedup and roughly halves rollout memory, both directly through faster low-precision kernels and indirectly through the larger batches a smaller memory footprint allows.

Standard RL algorithms assume trajectories are sampled from the current policy being optimized. When a quantized model (FP8/INT8) generates the rollout while the trainer stays at higher precision (BF16), the resulting precision gap introduces a \emph{rollout-training mismatch}~\citep{qiu2026fp8} that biases every gradient estimate. Because the policy gradient sums over token-level log-probability terms, small per-token discrepancies compound across long generation horizons, and on some reasoning benchmarks, the resulting bias is severe enough to collapse training accuracy outright: FP8 rollout drops AIME25 accuracy on Qwen3-8B by $6.63\%$.

Prior research has sought to bypass the mismatch problem by enforcing uniform precision across both rollout and training \citep{sglang2025fp8}. While this avoids distributional shift, it sacrifices the numerical integrity of the trainer. We instead embrace the common production paradigm of BF16 training with FP8 rollout, treating mismatch as an unavoidable yet manageable reality. Standard Importance Sampling (IS) \citep{pmlr-v97-hanna19a, yao2025flashrl} remains the primary defense for bias reduction, yet it often falls short in variance control, resulting in erratic training trajectories.

\begin{figure}[t]
    \centering
    \includegraphics[width=0.7\linewidth]{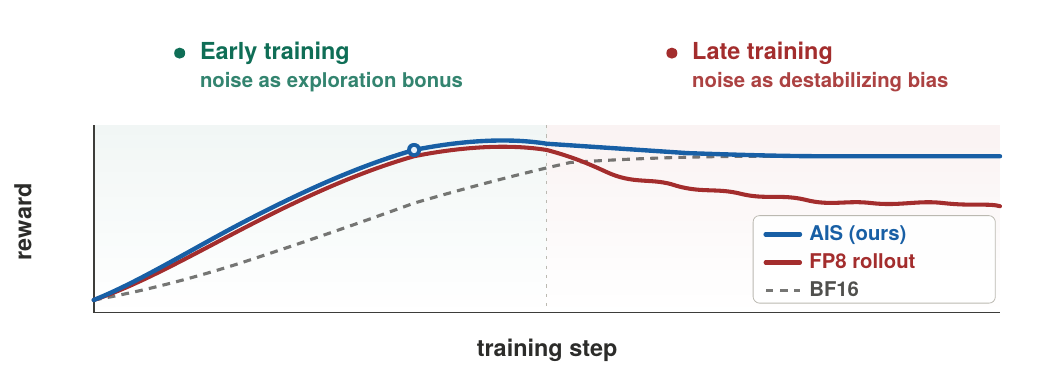}
    \caption{
        Reward trajectory of Qwen3-8B on GSM8K under FP8 rollout.
    }
    \label{fig:dual-role-mismatch}
    \vspace{-12pt}
\end{figure}

Crucially, we challenge the conventional wisdom that frames this mismatch solely as an adversary. As demonstrated in Figure~\ref{fig:dual-role-mismatch}, quantization mismatch operates as a double-edged sword: early in training, it serves as a stochastically driven exploration bonus, exposing the gradient to a broader behavioral set and preventing early stagnation. Yet, as training converges, the same noise devolves into a source of instability. This dual nature suggests that the challenge is not to eradicate the mismatch, but to govern it. The appropriate response is therefore not to eliminate the mismatch but to manage it, suppressing its destabilizing component while preserving its useful one, a balance that no fixed-threshold rule can sustain throughout training.

We propose Adaptive Importance Sampling (AIS), which targets exactly this balance. At each gradient step, AIS computes three real-time diagnostics from quantities the trainer already evaluates---\emph{weight reliability}, \emph{divergence severity}, and \emph{variance amplification}---and combines them into a single mixing coefficient that interpolates between the uncorrected gradient and a fully importance-weighted one. This coefficient tightens correction when the mismatch becomes destabilizing and relaxes it when the mismatch is benign, suppressing harmful variance while preserving the exploratory signal that quantized rollouts can provide. The mechanism reduces to standard training as the mismatch reduces, and adds negligible overhead because all three diagnostics are scalar reductions over statistics already computed during the policy update.

We integrate AIS into GRPO~\citep{shao2024deepseekmath} and evaluate on LLaDA-8B-Instruct~\citep{nie2025large}, Qwen3-8B, and Qwen3.5-9B~\citep{yang2025qwen3} across mathematical reasoning, planning, and general benchmarks. To our knowledge, this is the first study of low-precision rollout for RL on diffusion LLMs. AIS recovers most of the accuracy gap to BF16 across reasoning benchmarks for both architectures, substantially outperforming truncation sampling, while matching or improving over the BF16 reference on most tasks and retaining the $1.5$ to $2.76\times$ rollout speedup and roughly $50\%$ memory reduction of FP8.

%% file: sections/method_new.tex
\section{Preliminaries}
\newcommand{\oldtheta}{\theta_{\mathrm{old}}}
\newcommand{\qoldtheta}{\hat{\theta}_{\mathrm{old}}}

\subsection{LLM Quantization}
\label{sec:prelim_quantization}

Quantization reduces the numerical precision of model parameters and activations to enable faster inference and lower memory consumption. Given a full-precision tensor $\mathbf{X} \in \mathbb{R}^{n}$, quantization maps each element to a discrete set of representable values. For symmetric linear quantization to $b$ bits, the quantized representation is:
\begin{equation}
    Q(\mathbf{X}) = \text{clamp}\!\left(\left\lfloor \frac{\mathbf{X}}{s} \right\rceil,\; -2^{b-1},\; 2^{b-1}-1\right) \cdot s, \quad s = \frac{\max(|\mathbf{X}|)}{2^{b-1}-1},
    \label{eq:quantization}
\end{equation}
where $s$ is the scaling factor and $\lfloor \cdot \rceil$ denotes rounding to the nearest integer. The quantization error $\boldsymbol{\epsilon} = Q(\mathbf{X}) - \mathbf{X}$ introduces perturbations that propagate through the network's forward pass.
 
In this work, we focus on FP8 (E4M3) quantization~\citep{micikevicius2022fp8}, which represents values using 4 exponent bits and 3 mantissa bits. Unlike integer quantization, FP8 provides a wider dynamic range at the cost of reduced precision in the mantissa, making it particularly suitable for the diverse value distributions encountered in transformer activations. Modern accelerators (e.g., NVIDIA H100) provide native FP8 support, enabling up to $2\times$ throughput improvement over BF16 for matrix multiplications.

To accelerate the rollout phase, we quantize both the weights and activations of the rollout model $\oldtheta$ to obtain a low-bit counterpart $\qoldtheta = Q(\oldtheta)$. Because quantization affects the entire forward pass, the logits computed by the quantized model differ from those of the full-precision training model, even when both share the same underlying parameters:
\begin{equation}
    \pi_{\qoldtheta}(\cdot \mid x_{<t}) = \text{softmax}\!\big(f(x; \qoldtheta)\big) \;\neq\; \text{softmax}\!\big(f(x; \oldtheta)\big) = \pi_{\oldtheta}(\cdot \mid x_{<t}).
    \label{eq:logit_mismatch}
\end{equation}
This distributional divergence between $\pi_{\qoldtheta}$ and $\pi_{\oldtheta}$ is the root cause of the rollout-training mismatch analyzed in Section 3.
\subsection{GRPO and On-Policy Policy Gradients}
Group Relative Policy Optimization (GRPO)~\citep{shao2024deepseekmath} simplifies PPO~\citep{schulman2017proximal} by replacing the learned value function with group-level statistics.
For each query $q$, GRPO samples $G$ responses $\{o_1, \ldots, o_G\}$ from the old policy $\pi_{\text{old}}$ and computes the advantage for each response as
\begin{equation}
    A_i^k = r_i - \mathrm{mean}(\{r_j\}_{j=1}^G), \quad 1 \leq k \leq |o_i|,
    \label{eq:grpo_adv}
\end{equation}
where $r_i$ is the reward for $o_i$.
Following \citet{liu2025understanding}, we use the unnormalized form to avoid the bias introduced by dividing by a state-dependent standard deviation.
The GRPO objective combines PPO-style clipping with a reverse KL penalty:
\begin{equation}
\label{eq:grpo_loss}
\resizebox{0.93\columnwidth}{!}{
$
\displaystyle
\mathcal{J}(\theta) = \mathbb{E}_{\substack{q, \{o_i\}}}
    \left[ \frac{1}{G}\sum_{i=1}^G \frac{1}{|o_i|}\sum_{t=1}^{|o_i|} \min
    \left(\frac{\pi_{\theta}(o_{i,t}|q)}{\pi_{\text{old}}(o_{i,t}|q)} A_{i,t},\; \mathrm{clip}\!\left(\frac{\pi_{\theta}(o_{i,t}|q)}{\pi_{\text{old}}(o_{i,t}|q)}, 1{-}\alpha, 1{+}\alpha \right) A_{i,t}\right)
- \beta\, D_{\mathrm{KL}}\!\left(\pi_\theta \| \pi_\text{ref}\right)
\right],
$
}
\end{equation}
where $\pi_{\text{old}}$ is the sampling policy, $\pi_{\text{ref}}$ is the reference policy (typically the initial SFT model), $\alpha$ controls the clipping range, and $\beta$ governs KL regularization strength.

\section{Method}
\label{sec:method}

\begin{wrapfigure}{r}{0.5\textwidth}
    \centering
    \vspace{-10pt}
    \includegraphics[width=0.5\textwidth]{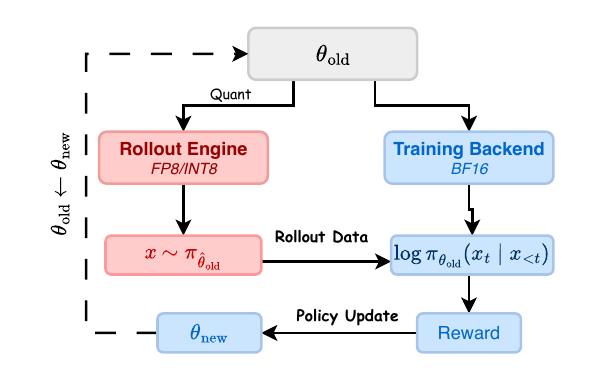}
    \caption{Quantized RL training pipeline. }
    \label{fig:quantized-rl-pipeline}
    \vspace{-10pt}
\end{wrapfigure}

\label{sec:problem-formulation}
Modern large-scale RL pipelines for LLMs decouple rollout generation from policy optimization to maximize hardware utilization~\citep{espeholt2018impala}. As illustrated in Figure~\ref{fig:quantized-rl-pipeline}, the rollout phase uses a dedicated inference engine (e.g., vLLM~\citep{kwon2023efficient}) to sample trajectories, while the training phase computes gradients and updates parameters using a separate backend (e.g., FSDP~\citep{zhao2023pytorch}). To accelerate rollout, we quantize the old actor $\oldtheta$ to a low-bit counterpart $\qoldtheta = Q(\oldtheta)$, enabling efficient matrix multiplication at reduced precision (FP8/INT8).

This design introduces a fundamental discrepancy. Although $\oldtheta$ and $\qoldtheta$ derive from the same parameters, quantization alters the forward pass, producing different logits and therefore different token-level probabilities:
\begin{equation}
    \pi_{\qoldtheta}(x_t \mid x_{<t}) \;\neq\; \pi_{\oldtheta}(x_t \mid x_{<t}).
    \label{eq:logit-mismatch}
\end{equation}

We refer to this as the \emph{rollout-training mismatch}. This mismatch stems not only from reduced numerical precision but also from broader implementation discrepancies between inference and training frameworks (e.g., disparate hardware backends or decoding strategies). These factors collectively exacerbate the divergence between the quantized rollout policy $\pi_{\qoldtheta}$ and the full-precision learner policy $\pi_{\oldtheta}$.

Since trajectories are sampled from the quantized policy $\pi_{\qoldtheta}$ while gradients are computed under the full-precision policy $\pi_{\oldtheta}$, the resulting policy gradient estimator
\begin{equation}
    \nabla_\theta J \approx \mathbb{E}_{x \sim \pi_{\qoldtheta}}
    \left[A(x)\,\nabla_\theta \log \pi_{\oldtheta}(x)\right]
    \label{eq:mismatch-estimator}
\end{equation}
where $A(x)$ denotes the advantage estimate associated with trajectory $x$, deviates from the desired on-policy gradient $\mathbb{E}_{x \sim \pi_{\oldtheta}}[A(x)\,\nabla_\theta \log \pi_{\oldtheta}(x)]$. The induced bias scales with $D_{\mathrm{KL}}(\pi_{\qoldtheta} \| \pi_{\oldtheta})$, which we observe grows monotonically during training as the policy sharpens and quantization artifacts interact with increasingly peaked distributions. This compounding effect degrades both sample efficiency and final performance~\citep{yao2025offpolicy}.

\paragraph{Importance sampling correction and its limitations.}
A common approach to mitigating the rollout-training mismatch is importance sampling (IS)~\citep{rubinstein2016simulation}. By reweighting each trajectory using the ratio between the full-precision learner policy $\pi_{\oldtheta}$ and the quantized rollout policy $\pi_{\qoldtheta}$, we can, in principle, recover the on-policy expectation:
\begin{equation}
    w(x) \;=\; \frac{\pi_{\oldtheta}(x)}{\pi_{\qoldtheta}(x)} \;=\; \prod_{t=1}^{T} \rho_t, \qquad \rho_t \;=\; \frac{\pi_{\oldtheta}(x_t \mid x_{<t})}{\pi_{\qoldtheta}(x_t \mid x_{<t})}.
    \label{eq:is-weight}
\end{equation}
However, standard IS suffers from prohibitively high variance in LLMs, as the cumulative weight $w(x)$ is a product of $T$ per-token ratios that can grow or shrink exponentially with sequence length, especially when $\pi_{\qoldtheta}$ and $\pi_{\oldtheta}$ diverge due to quantization. To control this variance, FlashRL~\citep{yao2025flashrl} employs Truncated Importance Sampling (TIS), which applies a per-token clip at a fixed threshold $C$ rather than accumulating across the trajectory. The resulting gradient estimator is
\begin{equation}
    \nabla_\theta J_{\mathrm{TIS}} \;\approx\; \mathbb{E}_{x \sim \pi_{\qoldtheta}} \!\left[ \sum_{t=1}^{T} \min(\rho_t,\, C)\, A_t \, \nabla_\theta \log \pi_{\oldtheta}(x_t \mid x_{<t}) \right],
    \label{eq:token-level-tis}
\end{equation}
where $A_t$ is the token-level advantage (e.g., the group-relative advantage in GRPO~\citep{shao2024deepseekmath}).

\begin{wrapfigure}{r}{0.5\textwidth}
    \centering
    \vspace{-10pt}
    \includegraphics[width=0.44\textwidth]{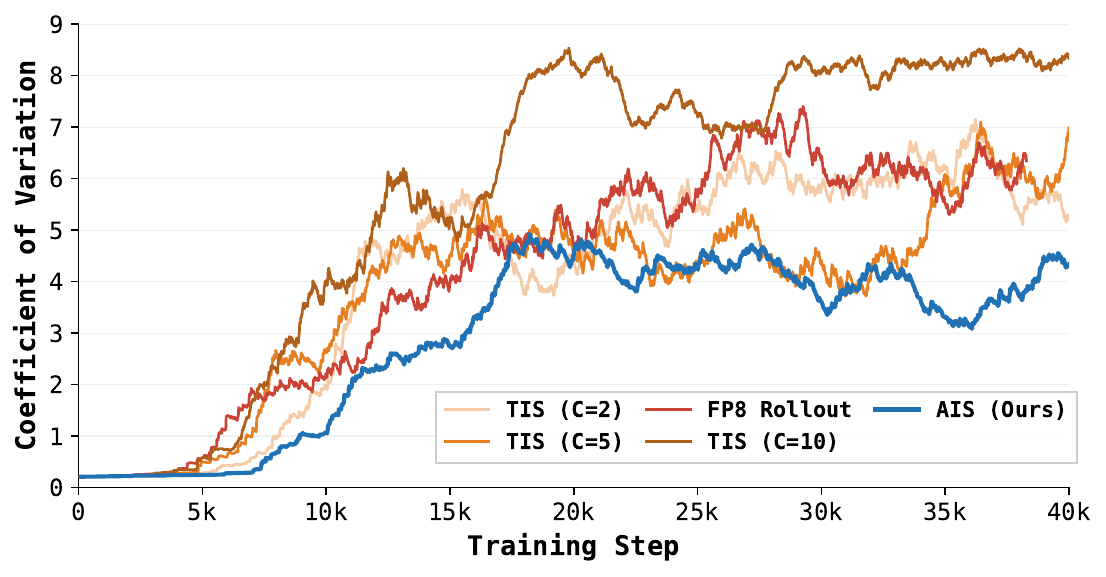}
    \caption{Variation analysis during LLaDA-8B training on GSM8K.}
    \label{fig:training-cv}
    \vspace{-10pt}
\end{wrapfigure}

While TIS stabilizes training over uncorrected FP8 rollout, applying the same correction strength to every batch is limited. Figure~\ref{fig:training-cv} sweeps the TIS threshold $C$ and tracks the Coefficient of Variation $\mathrm{CV} = \sigma(w)/\mu(w)$ of the importance weights: $C{=}10$ lets CV grow rapidly, $C{=}2$ over-suppresses the corrective signal, and $C{=}5$ sits between the two. As the policy sharpens, no single $C$ remains right. AIS keeps the same truncation as a stable variance cap but introduces a per-batch coefficient $\alpha(x)$ (Section~\ref{sec:method}.1) that controls \emph{how strongly} the truncated correction is applied, holding CV lowest throughout training.

\begin{figure}[t]
\centering
\includegraphics[width=\linewidth]{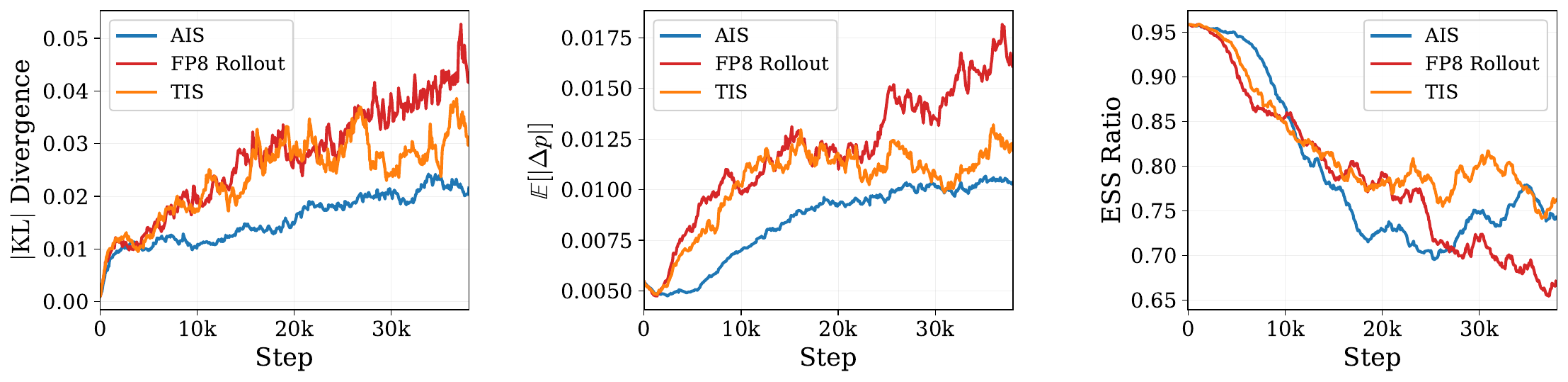}
\caption{LLaDA-8B training behavior under different low-bit rollout strategies on the GSM8K. }
\label{fig:training-dynamics}
\end{figure}

\subsection{Adaptive Importance Sampling for Quantized Rollout}
\label{sec:ais}
As illustrated in Figure~\ref{fig:training-dynamics}, the rollout-training mismatch is inherently non-stationary. The three panels track complementary diagnostics of the mismatch over training for FP8 rollout (no correction), truncated importance sampling (TIS) with a fixed threshold, and our AIS: (i) the KL divergence $|\mathrm{KL}(\pi_{\mathrm{roll}} \,\|\, \pi_{\mathrm{train}})|$ (left), capturing the overall distributional gap, (ii) the expected per-token probability discrepancy $\mathbb{E}[|\Delta p|]$ (middle), which instantiates the mismatch signal $\bar{d}$ in Eqn.~\ref{eq:discrepancy}, and (iii) the Effective Sample Size (ESS) ratio (right), which underlies the reliability signal $\alpha_{\text{ess}}$ in Eqn.~\ref{eq:alpha-ess}. Across all three metrics, the uncorrected FP8 rollout steadily drifts away from the learner as distributional sharpening intensifies, with its ESS ratio decaying from above $0.95$ to below $0.70$. TIS partially mitigates this drift but still trails AIS on all three panels, reflecting its inability to adapt as the mismatch evolves. Such dynamics underscore a fundamental limitation of treating every batch with the same correction strength. This observation motivates an \emph{Adaptive Importance Sampling} (AIS) framework that dynamically navigates the bias-variance frontier.

\paragraph{The Oracle-Optimal Mixing Coefficient.}
We formalize the adaptive correction as an interpolation between the uncorrected gradient ($\alpha = 0$) and the importance-weighted gradient ($\alpha = 1$). Consider the mixed estimator:
\begin{equation}
    \hat{g} \;=\; \mathbb{E}_{x \sim \pi_{\mathrm{rollout}}}\!\left[\bigl(1 - \alpha(x) + \alpha(x)\,\bar{w}(x)\bigr) A(x)\, \nabla_\theta \log \pi_{\mathrm{train}}(x)\right],
    \label{eq:ais-estimator}
\end{equation}
\begin{equation}
    w(x) = \frac{\pi_{\mathrm{train}}(x)}{\pi_{\mathrm{rollout}}(x)}, \qquad \bar{w}(x) = \min\!\bigl(w(x),\, C\bigr), \qquad \alpha(x) \in [0,1],
    \label{eq:ais-components}
\end{equation}
where $\bar{w}(x)$ is the truncated importance weight with threshold $C \geq 1$, and $\alpha(x) \in [0,1]$ is a per-sample mixing coefficient that gates the correction strength: it is pushed toward $1$ under severe rollout-training mismatch and toward $0$ when the sample is already aligned. Note that $w(x) \to 1$ drives the coefficient $1 - \alpha(x) + \alpha(x)\,\bar{w}(x) \to 1$ regardless of $\alpha(x)$, so $\hat{g}$ automatically reduces to the uncorrected estimator in the absence of mismatch. Let $g = \mathbb{E}_{x \sim \pitrain}[A(x)\,\nabla_\theta \log \pitrain(x)]$ be the true on-policy gradient.The uncorrected estimator ($\alpha = 0$) incurs a bias $b_0 = \mathbb{E}_{x \sim \pi_{\mathrm{rollout}}}[A(x)\,\nabla_\theta\log\pi_{\mathrm{train}}(x)] - g$, while the IS estimator ($\alpha = 1$) is approximately unbiased but has variance $\sigma_1^2$ that scales with $\mathbb{E}[\bar{w}(x)^2]$. Treating $\alpha$ as a scalar for analytical tractability, the bias-variance decomposition reduces to a surrogate MSE of $(1-\alpha)^2 \|b_0\|^2 + \alpha^2 \sigma_1^2$. Minimizing this expression with respect to $\alpha$ yields:
\begin{equation}
    \alpha^* = \frac{\|b_0\|^2}{\|b_0\|^2 + \sigma_1^2}.
    \label{eq:oracle-alpha}
\end{equation}
This oracle coefficient provides a principled target: correction should be intensified only when the bias reduction outweighs the induced variance. Since $\|b_0\|$ and $\sigma_1^2$ are not directly observable during training, we derive a practical approximation through a factorized gating mechanism.

\paragraph{Decoupled Modulation via Observable Signals.}
We approximate $\alpha^{*}$ by three observable signals that capture, respectively, the \emph{reliability} of importance weights, the \emph{severity} of the rollout-training divergence, and the \emph{amplification} of gradient variance.

\medskip
\noindent\textbf{(i) Weight Reliability ($\alpha_{\text{ess}}$).}
The variance term $\sigma_1^{2}$ is dominated by the concentration of importance weights: a diffuse weight distribution indicates that only a few samples carry most of the corrective signal, rendering the IS estimate unreliable. We quantify this through the Effective Sample Size (ESS) ratio and take its square root to align the signal with the standard deviation of the IS estimator:
\begin{equation}
    \alpha_{\text{ess}} \;=\; \sqrt{\mathrm{ESS}_{\text{ratio}}} \;=\; \bigl(1 + \mathrm{CV}^{2}(\bar{w})\bigr)^{-1/2},
    \label{eq:alpha-ess}
\end{equation}
where $\mathrm{CV}(\bar{w})$ denotes the sample coefficient of variation of $\bar{w}(x)$ over the current mini-batch. The signal approaches $1$ when weights are near-uniform and decays smoothly as the weight distribution becomes concentrated.


\medskip
\noindent\textbf{(ii) Divergence Severity ($\alpha_{\text{mis}}$).}
The ESS ratio is blind to whether correction is actually needed: when $\pi_{\mathrm{rollout}} \approx \pi_{\mathrm{train}}$, IS correction is redundant even with well-conditioned weights. We therefore measure the raw rollout-training divergence via the mean absolute log-probability discrepancy over the current mini-batch $\mathcal{B}$,
\begin{equation}
    \bar{d} \;=\; \frac{1}{|\mathcal{B}|}\sum_{x_t \in \mathcal{B}} \bigl|\log \pi_{\mathrm{train}}(x_t \mid x_{<t}) - \log \pi_{\mathrm{rollout}}(x_t \mid x_{<t})\bigr|,
    \label{eq:discrepancy}
\end{equation}
where $\mathcal{B}$ collects all token positions in the current mini-batch. We then gate the correction strength by a saturating function of $\bar{d}$:
\begin{equation}
    \alpha_{\text{mis}} \;=\; \min\!\left(1,\; \frac{\bar{d}}{\delta}\right),
    \label{eq:alpha-mis}
\end{equation}
where $\delta > 0$ is a saturation threshold. Correction is thus smoothly bypassed when the sampler and learner are well aligned and engaged only once the mismatch exceeds $\delta$.

\medskip
\noindent\textbf{(iii) Variance Amplification ($\alpha_{\text{var}}$).}
Weight-level diagnostics can still underestimate the variance cost when large importance weights coincide with large-magnitude advantages, inflating the gradient variance beyond what $\alpha_{\text{ess}}$ alone would predict. We monitor the variance-amplification ratio
\begin{equation}
    \Delta\sigma \;=\; \frac{\mathrm{std}\bigl(A(x) \cdot \bar{w}(x)\bigr)}{\mathrm{std}\bigl(A(x)\bigr) + \epsilon},
\end{equation}
where $\mathrm{std}(\cdot)$ denotes the sample standard deviation over the current mini-batch. We activate a subtractive penalty whenever $\Delta\sigma$ exceeds a tolerance $\gamma$:
\begin{equation}
    \alpha_{\text{var}} \;=\; \max\!\left(0,\; \frac{\Delta\sigma - \gamma}{\gamma}\right).
    \label{eq:alpha-var}
\end{equation}
This safeguard downweights the correction precisely in the regime where gradient-space noise, rather than weight-space concentration, becomes the dominant source of error.

\paragraph{Bilateral Factorization and Implementation.}
We synthesize the final mixing coefficient by factorizing it into a reliability gate and a necessity gate:
\begin{equation}
    \alpha = \underbrace{\operatorname{clip}\!\left(\alpha_{\text{ess}} - \beta \cdot \alpha_{\text{var}},\, 0,\, 1\right)}_{\text{Reliability}} \;\cdot\; \underbrace{\alpha_{\text{mis}}}_{\text{Necessity}},
    \label{eq:alpha-final}
\end{equation}
where $\beta$ modulates the sensitivity to gradient-space noise. This product form ensures that correction collapses to zero if either the weights are degenerate or the mismatch is negligible. In practice, AIS is implemented via a simple transformation of the advantage:
\begin{equation}
    \tilde{A}(x) = \left[1 + \alpha \cdot (\bar{w}(x) - 1)\right] A(x)
    \label{eq:adjusted-advantage}
\end{equation}
This formulation seamlessly integrates with standard RLHF frameworks while incurring negligible computational overhead.

\begin{algorithm}[t]
\caption{AIS-GRPO: Adaptive Importance Sampling for GRPO}
\label{alg:ais-grpo}
\begin{algorithmic}[1]
\Require Policy $\pi_\theta$, rollout engine $\pi_{\mathrm{roll}}$, reference policy $\pi_{\mathrm{ref}}$, clipping $\varepsilon$, IS truncation $C$, gating parameters $\delta, \gamma, \beta_{\mathrm{var}}$
\State Initialize $\theta_{\mathrm{old}} \gets \theta$
\For{each training iteration}
    \State \textbf{Rollout Sampling:}
    \State \quad Sample prompts $\{q^{(j)}\} \sim \mathcal{D}$
    \State \quad Generate $G$ responses per prompt using quantized $\pi_{\mathrm{roll}}(\cdot \mid \theta_{\mathrm{old}})$
    \State \quad Compute rewards $R(q, o_i)$ and group-relative advantages $\hat{A}_i^t$
    \State \textbf{Adaptive Coefficient Estimation:}
    \State \quad Compute full-precision log-probs $\log \pitrain(o_i^t \mid \theta_{\mathrm{old}})$ for all generated tokens
    \State \quad $\rho_{\mathrm{rt}, i}^t \gets \exp \bigl( \log \pitrain(o_i^t \mid \theta_{\mathrm{old}}) - \log \piroll(o_i^t \mid \theta_{\mathrm{old}}) \bigr)$ \Comment{Sampler-Trainer ratio}
    \State \quad Compute $\{\alpha_{\mathrm{ess}}, \alpha_{\mathrm{mis}}, \alpha_{\mathrm{var}}\}$ via Eq.~\eqref{eq:alpha-ess}--\eqref{eq:alpha-var}
    \State \quad $\alpha \gets \operatorname{clip}(\alpha_{\mathrm{ess}} - \beta_{\mathrm{var}} \cdot \alpha_{\mathrm{var}}, 0, 1) \cdot \alpha_{\mathrm{mis}}$ \Comment{Bilateral gating, Eq.~\eqref{eq:alpha-final}}
    \State \textbf{Advantage Rectification:}
    \State \quad $\tilde{w}_i^t \gets 1 + \alpha \cdot (\min(\rho_{\mathrm{rt},i}^t, C) - 1)$ \Comment{Rectification weight}
    \State \quad $\tilde{A}_i^t \gets \tilde{w}_i^t \cdot \hat{A}_i^t$ \Comment{AIS-adjusted advantage, Eq.~\eqref{eq:adjusted-advantage}}
    \State \textbf{Policy Update:}
    \State \quad Update $\theta$ by maximizing $\mathcal{L}_{\mathrm{AIS\text{-}GRPO}}$ in Eq.~\eqref{eq:ais-grpo} using $\tilde{A}_i^t$
    \State \quad Update $\theta_{\mathrm{old}} \gets \theta$
\EndFor
\end{algorithmic}
\end{algorithm}

\subsection{Integration with GRPO}
\label{sec:ais-grpo}

We instantiate AIS within Group Relative Policy Optimization (GRPO)~\citep{shao2024deepseekmath}, a predominant on-policy RL algorithm for reasoning LLMs. A critical distinction in AIS-GRPO is the presence of two \textit{orthogonal} importance ratios that address disparate sources of distributional shift:

\begin{itemize}[leftmargin=*, labelsep=0.5em]
    \item \textbf{Rollout--Training Ratio} ($\rho_{\text{rt},i}^t$): Corrects the precision-induced shift between the quantized sampler and the full-precision learner \textit{at the same policy checkpoint} ($\theta_{\text{old}}$).
    \item \textbf{Old--New Policy Ratio} ($\rho_{\text{on},i}^t$): Controls the standard trust-region divergence between the behavioral policy $\pi_{\theta_{\text{old}}}$ and the current target $\pi_\theta$ \textit{during optimization}.
\end{itemize}

By applying the adaptive coefficient $\alpha$ (Eq.~\ref{eq:alpha-final}), we transform the rollout--training ratio into a rectification weight $\tilde{w}_i^t = 1 + \alpha (\min(\rho_{\text{rt},i}^t, C) - 1)$. The resulting \textbf{AIS-GRPO objective} is formulated as:
\begin{equation}
\label{eq:ais-grpo}
\mathcal{L}_{\text{AIS-GRPO}}(\theta) = \mathbb{E}_{q, \{o_i\} \sim \piroll} \left[ \frac{1}{G}\sum_{i=1}^G \frac{1}{|o_i|}\sum_{t=1}^{|o_i|} \tilde{w}_i^t \cdot \min\left( \rho_{\text{on},i}^t \hat{A}_i^t, \text{clip}(\rho_{\text{on},i}^t, 1\pm\varepsilon) \hat{A}_i^t \right) \right] - \beta \mathcal{D}_{\text{KL}},
\end{equation}
where $\mathcal{D}_{\text{KL}}$ denotes the standard KL divergence against the reference policy. This formulation achieves a clean \textit{separation of concerns}: $\tilde{w}_i^t$ rectifies the numerical artifacts of the rollout engine, while the clipped surrogate objective maintains optimization stability. Notably, when the quantization mismatch is negligible ($\alpha \to 0$), the AIS weight $\tilde{w}_i^t$ converges to unity, and the objective reduces exactly to vanilla GRPO, ensuring seamless backward compatibility and minimal overhead.

%% file: sections/experiments.tex
\section{Experiments}
\label{sec:experiments}

We design experiments to investigate four questions:
\textbf{(Q1)} Does the rollout--learner mismatch induced by FP8 quantization materially degrade RL training quality?
\textbf{(Q2)} Can AIS recover the lost performance while preserving the efficiency benefits of low-precision rollout?
\textbf{(Q3)} Does AIS generalize across model architectures, including non-autoregressive diffusion models?
\textbf{(Q4)} What are the wall-clock and memory savings from FP8 rollout in practice?

\subsection{Models, Tasks, and Setup}

\paragraph{Models.} We evaluate on two architecture families: autoregressive Qwen3-8B and Qwen3.5-9B~\citep{yang2025qwen3}, and the masked diffusion model LLaDA-8B-Instruct~\citep{nie2025large}. All models are trained with GRPO under both full-parameter and LoRA~\citep{hu2022lora} fine-tuning.

\paragraph{Tasks.} We cover two categories of reasoning tasks: \emph{mathematical reasoning} (GSM8K~\citep{cobbe2021training}, MATH500~\citep{hendrycks2021measuring}, AIME 2025, AMC 2023, Olympiad Bench~\citep{olympiadbench}) and \emph{planning} ($4\!\times\!4$ Sudoku and Countdown). We also report general-capability benchmarks (MMLU, PIQA, ARC, HellaSwag, WinoGrande) as a sanity check.

\paragraph{Others.} We mainly compare four configurations: \textbf{BF16}, \textbf{FP8 Rollout}, \textbf{FlashRL (TIS)}~\citep{yao2025flashrl}, and \textbf{AIS (ours)}. FP8 is applied only to the rollout policy, and the trainer and optimizer state remain in BF16. Full training, evaluation, and dataset details are in Appendix~\ref{appendix:training_details}.

\subsection{Main Results on Autoregressive Models}
\label{sec:main-ar}

We first establish the core results of AIS on autoregressive LLMs by reporting reasoning accuracy (Table~\ref{tab:ar-reasoning}) and general-capability accuracy (Table~\ref{tab:ar-general} in Appendix) on Qwen3-8B and Qwen3.5-9B.

\begin{table}[t]
\centering
\small
\caption{Evaluation accuracy (\%) on reasoning benchmarks for autoregressive models.}
\resizebox{\textwidth}{!}{%
\begin{tabular}{ll c cccccc}
\toprule
\textbf{Model} & \textbf{Method} & \textbf{Bitwidth} & \textbf{GSM8K} & \textbf{MATH500} & \textbf{AIME25} & \textbf{AMC} & \textbf{Olympiad} & \textbf{Sudoku} \\
\midrule
\multirow{4}{*}{Qwen3-8B}
  & RL (BF16)        & BF16 & 89.84 & 76.80 & 36.70 & 68.50 & 48.80 & 38.28 \\
\cmidrule(lr){2-9}
  & RL (FP8 Rollout) & FP8  & 87.65 & 75.80 & 26.81 & 61.25 & 42.10 & 36.33 \\
  & FlashRL (TIS)    & FP8  & 88.63 & 76.20 & 34.23 & 63.40 & 46.50 & 37.50 \\
  & AIS (Ours)       & FP8  & \textbf{89.99} & \textbf{77.00} & \textbf{43.33} & \textbf{69.75} & \textbf{48.20} & \textbf{39.10} \\
\midrule
\multirow{4}{*}{Qwen3.5-9B}
  & RL (BF16)        & BF16 & \textbf{92.57} & 83.20 & 69.50 & 74.20 & \textbf{51.30} & \textbf{89.84} \\
\cmidrule(lr){2-9}
  & RL (FP8 Rollout) & FP8  & 90.22 & 71.00 & 60.00 & 65.80 & 45.60 & 85.94 \\
  & FlashRL (TIS)    & FP8  & 90.80 & 72.30 & 66.67 & 67.50 & 49.02 & 88.67 \\
  & AIS (Ours)       & FP8  & 91.74 & \textbf{85.60} & \textbf{73.33} & \textbf{75.50} & 50.80 & 89.45 \\
\bottomrule
\end{tabular}%
}
\label{tab:ar-reasoning}
\end{table}

\paragraph{FP8 rollout severely degrades hard reasoning (Q1).}
Naive FP8 rollout inflicts substantial degradation on difficulty-sensitive benchmarks. As shown in Table~\ref{tab:ar-reasoning}, AIME25 accuracy on Qwen3-8B drops from $36.70\%$ to $26.81\%$ ($-9.89\%$), and on Qwen3.5-9B, MATH500 and AIME25 fall by $12.20\%$ and $9.50\%$, respectively. The degradation is markedly milder on easier benchmarks such as GSM8K ($-2.19\%$ and $-2.35\%$ on the two models), consistent with our analysis that the rollout-training divergence compounds with with task difficulty.

\paragraph{Static correction provides only partial recovery.}
FlashRL (TIS) recovers part of the lost accuracy via truncated importance ratios but cannot track the non-stationary mismatch with a single correction strength. On Qwen3-8B AIME25, TIS lifts accuracy from $26.81\%$ to $34.23\%$, still $2.47\%$ below BF16. Similar gaps persist on AMC and Olympiad, indicating that static correction is insufficient whenever the rollout-training divergence drifts during training.

\paragraph{AIS recovers and often surpasses full-precision performance (Q2).}
AIS attains the highest FP8 accuracy on all reasoning benchmarks for both models, matching or exceeding BF16 on $11$ of $12$ pairs in Table~\ref{tab:ar-reasoning}, with the only shortfalls (Olympiad on Qwen3-8B, Sudoku on Qwen3.5-9B) sitting within $1\%$ of BF16. The recovery is most pronounced on hard benchmarks, with AIS lifting AIME25 from $26.81\%$ to $43.33\%$ on Qwen3-8B and from $60.00\%$ to $73.33\%$ on Qwen3.5-9B, exceeding BF16 in both cases. We attribute occasional gains beyond BF16 to stochastic FP8 perturbations acting as exploration noise, leaving a precise characterization to future work.

\paragraph{General capabilities are preserved.}
On six general-capability benchmarks (Table~\ref{tab:ar-general}, Appendix), AIS matches BF16 on every task for both models, indicating that adaptive correction does not overfit to the RL reward signal.

\subsection{Generalization to Diffusion-based Architectures}
\label{sec:generalization-dllm}

Having established AIS on autoregressive models, we next ask whether its benefits extend to non-autoregressive ones (Q3). We evaluate on LLaDA-8B-Instruct, a masked diffusion model. To our knowledge, this is the first study of low-bit rollout for RL on diffusion LLMs.
\begin{table*}[t]
\centering
\small
\setlength{\tabcolsep}{4pt}
\renewcommand{\arraystretch}{1.1}
\caption{Full-model fine-tuning performance on LLaDA-8B across reasoning benchmarks.}
\begin{tabular}{l c c c c c c c c c c c c}
\toprule
\multirow{2}{*}{\textbf{Method}}
& \multicolumn{3}{c}{\textbf{GSM8K}}
& \multicolumn{3}{c}{\textbf{MATH500}}
& \multicolumn{3}{c}{\textbf{Countdown}}
& \multicolumn{3}{c}{\textbf{Sudoku}} \\
\cmidrule(lr){2-4}\cmidrule(lr){5-7}\cmidrule(lr){8-10}\cmidrule(lr){11-13}
& \textbf{128} & \textbf{256} & \textbf{512}
& \textbf{128} & \textbf{256} & \textbf{512}
& \textbf{128} & \textbf{256} & \textbf{512}
& \textbf{128} & \textbf{256} & \textbf{512} \\
\midrule
\textbf{BF16}
& 74.60 & \textbf{81.10} & \textbf{82.10}  & \textbf{33.80} & \textbf{38.91} & 40.20  & 44.80 & 52.00 & 52.20  & 22.10 & \textbf{26.70} & \textbf{29.25} \\
\textbf{FP8 Rollout}
& 73.99 & 79.15 & 80.89  & 30.80 & 37.00 & 37.80  & 37.73 & 34.38 & 33.20  & 21.63 & 23.30 & 23.01 \\
\textbf{FlashRL (TIS)}
& 74.37 & 78.24 & 80.44  & 33.68 & 37.40 & 39.00  & 39.06 & 42.97 & 48.05  & 23.80 & 24.69 & 23.73 \\
\textbf{AIS (Ours)}
& \textbf{74.76} & 80.52 & 81.94  & 32.80 & 38.60 & \textbf{40.27}  & \textbf{54.30} & \textbf{57.03} & \textbf{54.69}  & \textbf{23.88} & 25.07 & 24.07 \\
\bottomrule
\end{tabular}
\label{tab:bf16-fp8-tis-ais}
\end{table*}

\paragraph{Consistent degradation under FP8 rollout.}
The pattern observed on autoregressive models persists on LLaDA. Under full-model fine-tuning (Table~\ref{tab:bf16-fp8-tis-ais}), FP8 rollout underperforms BF16 on every task, with the gap widening as the reasoning horizon grows: $-$$3.00\%$ on MATH500 at $128$ tokens and $-$$19.00\%$ on Countdown at $512$ tokens. Under LoRA (Table~\ref{tab:main_benchmark_lora_llada}, Appendix), the gap widens further on MATH500. This suggests quantization-induced mismatch is not specific to autoregressive decoding.

\paragraph{AIS generalizes across architectures.}
AIS recovers BF16 accuracy across both fine-tuning regimes on most (task, length) configurations. On Countdown under full-model tuning, AIS reaches $54.30\% / 57.03\% / 54.69\%$ at $128 / 256 / 512$ tokens, comparable to or above BF16 ($44.80\% / 52.00\% / 52.20\%$). Under LoRA, AIS likewise recovers or matches BF16 on the majority of configurations (Table~\ref{tab:main_benchmark_lora_llada}, Appendix). The consistency of these results across the sequential paradigm of Qwen and the iterative denoising paradigm of LLaDA suggests that the reliability--necessity factorization in AIS transfers across distinct generation paradigms.

\begin{figure}[t]
    \centering
    \includegraphics[width=\linewidth]{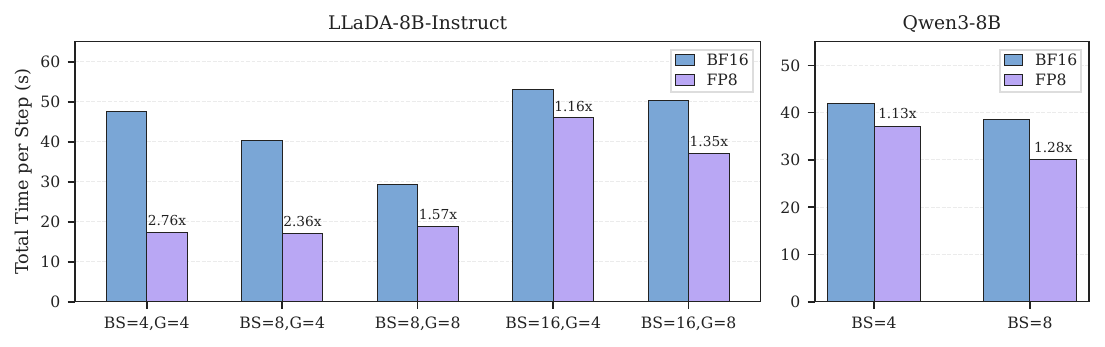}
    \caption{End-to-end rollout speedup from FP8 quantization. Left: LLaDA-8B-Instruct across five (batch size, generation) configurations. Right: Qwen3-8B across two batch sizes.}
    \label{fig:rollout-speedup}
\end{figure}

\subsection{Rollout Efficiency}
\label{sec:efficiency}

We conclude by quantifying the system-level benefits of FP8 rollout (Q4), which is ultimately the motivation for tolerating the distributional mismatch addressed above. While AIS recovers accuracy across both architectures (Section~\ref{sec:generalization-dllm}), the magnitude of the underlying efficiency gain depends on the generation paradigm.

\paragraph{End-to-end speedup.}
Figure~\ref{fig:rollout-speedup} reports end-to-end rollout speedup. On LLaDA-8B, FP8 rollout reaches $2.76\times$ at the small-batch configuration (BS$=4$, G$=4$), where the generation stage dominates wall-clock time and benefits most from reduced-precision arithmetic; as batch size grows and memory bandwidth becomes the bottleneck, the speedup narrows to $1.16$--$1.35\times$ at BS$=16$. On Qwen3-8B, FP8 yields a more modest $1.13\times$ and $1.28\times$ at BS$=4$ and BS$=8$ respectively. The difference reflects the generation paradigm: autoregressive decoding is dominated by per-token KV-cache memory bandwidth, whereas iterative denoising in diffusion models performs full-sequence matrix multiplications at each step, where FP8 tensor-core kernels deliver larger compute gains. The per-step latency breakdown is provided in Appendix~\ref{appendix:bf16-fp8-benchmark}.

\paragraph{Memory efficiency.}
FP8 rollout halves the rollout model footprint relative to BF16 ($7.95$\,GB vs.\ $\sim$$15.9$\,GB on LLaDA-8B; $7.45$\,GB vs.\ $\sim$$14.9$\,GB on Qwen3-8B), enabling larger effective batch sizes within the same memory budget. For instance, BF16 OOMs at BS$=16$ with G$=8$ on LLaDA-8B, whereas FP8 completes the same configuration (Appendix Table~\ref{tab:bf16-fp8-benchmark}).

\section{Conclusion}

In this work, we introduce Adaptive Importance Sampling (AIS), a correction framework that enables low-precision RL training of LLMs without sacrificing model quality. The starting point of AIS is the observation that the rollout-training mismatch is non-stationary and not uniformly harmful: applying the same correction strength across training stages either clips useful exploration or admits destabilizing bias. AIS instead reads three real-time diagnostics from each batch, weight reliability, divergence severity, and variance amplification, and combines them into a single mixing coefficient that adjusts correction strength on the fly, suppressing the destabilizing component while preserving the exploratory benefit. AIS reduces to the uncorrected estimator when no significant shift is detected. To our knowledge, this is the first study of low-precision rollout in RL training of diffusion-based LLMs. AIS matches or surpasses BF16 accuracy while retaining a $1.5\text{--}2.76\times$ rollout speedup and roughly $48\%$ memory reduction.

\newpage
\clearpage

%% file: sections/appendix.tex
\section{Appendix}

\section{Related Work}

\paragraph{Quantization of LLMs.}
Quantization is a standard technique for compressing large language models (LLMs)~\citep{gholami2022survey, nn2}, with $8$-bit formats offering a favorable trade-off between compression and accuracy under native hardware support. For $\mathrm{INT}8$ quantization, the main challenge is activation outliers. LLM.int8()~\citep{dettmers2022gpt3} keeps outlier channels in $\mathrm{FP}16$ while quantizing the rest to $\mathrm{INT}8$. SmoothQuant~\citep{xiao2023smoothquant} migrates the difficulty from activations to weights via per-channel scaling for end-to-end $\mathrm{W}8\mathrm{A}8$ inference. QLLM~\citep{liu2023qllm} further redistributes activation magnitudes through adaptive channel reassembly. On the training side, LLM-QAT~\citep{liu2023llm} performs data-free quantization-aware training to support $\mathrm{W}8\mathrm{A}8$ and lower-bit settings. $\mathrm{FP}8$ quantization replaces integer grids with floating-point formats (E4M3 and E5M2) standardized by \citet{micikevicius2022fp8}, whose wider dynamic range handles outliers more gracefully than $\mathrm{INT}8$. Both $\mathrm{FP}8$ PTQ~\citep{kuzmin2022fp8} and end-to-end $\mathrm{FP}8$ training~\citep{peng2023fp8lm} have been shown to match $\mathrm{FP}16$ accuracy on LLMs.

\paragraph{Infrastructure and Quantization for Efficient RL.}
Scaling RL on LLMs imposes heavy demands on both system and numerical efficiency. On the systems side, modern frameworks decouple rollout from training to improve utilization: OpenRLHF~\citep{hu2024openrlhf} and veRL~\citep{sheng2024hybridflow} place the rollout engine and learner on separate hardware groups with flexible role remapping, while AReaL~\citep{fu2025areal} pushes further toward full asynchrony by overlapping rollout and training through pipelined execution and stale-tolerant updates. Rollout itself is typically served by high-throughput inference engines such as vLLM~\citep{kwon2023efficient}. On the numerical side, quantization compresses LLMs into low-bit formats with native hardware support. For $\mathrm{INT}8$, the main challenge is activation outliers, addressed by LLM.int8()~\citep{dettmers2022gpt3} via mixed-precision retention and by SmoothQuant~\citep{xiao2023smoothquant} via per-channel scaling for end-to-end $\mathrm{W}8\mathrm{A}8$ inference. $\mathrm{FP}8$ quantization~\citep{micikevicius2022fp8} replaces integer grids with floating-point formats (E4M3, E5M2) whose wider dynamic range handles outliers more gracefully, and has been shown to match $\mathrm{FP}16$ accuracy in both PTQ~\citep{kuzmin2022fp8} and end-to-end training~\citep{peng2023fp8lm}.

\textbf{Infrastructure and Systems for RL Training. }
RLVR training places heavy demands on distributed systems, as each step interleaves long-context rollouts, reward computation, and policy updates across heterogeneous workloads. Early frameworks such as TRL~\citep{vonwerra2022trl} and TRLX~\citep{havrilla2023trlx} co-locate rollout and training on the same workers, leaving accelerators idle during decoding. To address this, OpenRLHF~\citep{hu2024openrlhf} and NeMo-Aligner~\citep{shen2024nemo} decouple rollout and trainer onto separate hardware groups, while veRL~\citep{sheng2024hybridflow} dynamically remaps devices between the two roles. A parallel line accelerates rollout via high-throughput inference engines such as vLLM~\citep{kwon2023efficient} and SGLang~\citep{zheng2024sglang}, exploiting continuous batching and prefix caching. More recent systems pursue full asynchrony, with AReaL~\citep{fu2025areal} and StreamRL~\citep{zhong2025streamrl} overlapping rollout and training through pipelined execution and stale-tolerant updates. Jointly optimizing rollout throughput, update staleness, and final model quality remains an open challenge.

\section{Experiment Details}
\label{appendix:training_details}

This appendix describes the training and evaluation protocols used in our experiments. Our LLaDA setup follows the hyperparameter recipe of~\citet{zhao2025d1} to ensure direct comparability, extending it with an FP8 rollout engine and the AIS correction module. Our Qwen experiments adapt the same pipeline to autoregressive architectures and incorporate evaluation components from DeepScaler~\citep{deepscaler2025} and \texttt{lm-evaluation-harness}~\citep{gao2021framework}. Full implementation details, including configuration files and launch scripts, are provided in our released codebase.

\subsection{RL Training Setup}
\label{appendix:training_infra}

\paragraph{Baselines.}
We compare against three configurations:
(1)~\textbf{BF16 (Original)}, the full-precision GRPO training pipeline serving as the upper-bound reference;
(2)~\textbf{FP8 Rollout}, which na\"ively quantizes the rollout model to FP8 without any distributional correction;
(3)~\textbf{FlashRL (TIS)}~\citep{liu2025flashrl}, representing \emph{static correction} via token-level importance sampling with a fixed truncation threshold.
Our proposed method, \textbf{AIS} (Adaptive Importance Sampling), represents \emph{adaptive correction} by replacing the fixed truncation of TIS with a per-batch adaptive mechanism (Section~\ref{sec:ais}).

\paragraph{Training and evaluation.}
We train a separate model for each task using a composed reward function combining formatting and correctness signals. Rollout sequence length is capped at 256/512/1024 tokens throughout RL training. FP8 quantization is applied only to the rollout model, and the reference model and optimizer states remain in BF16. All results was measured on 8$\times$ devices with 1,980 peak TFLOPS each. We evaluate all models with 0-shot prompting and greedy decoding, report accuracy (\%) on held-out test sets, and select the best checkpoint across evaluations conducted every 100 steps starting from step 600.

\paragraph{Framework.}
Our codebase extends the TRL library~\citep{vonwerra2022trl} with dedicated trainers for mismatched-precision RL, supporting both a BF16 baseline (rollout and learner both in BF16) and an FP8-rollout configuration paired with the AIS correction module.

\paragraph{Training modes.}
We evaluate under two regimes to probe robustness across scaling strategies:
\begin{itemize}[leftmargin=2em, itemsep=0.1em]
    \item \textbf{LoRA}~\citep{hu2022lora}: $r=128$ for LLaDA-8B; $r=64$ for Qwen3.5-9B. LoRA is applied to all linear projections (\textit{q, k, v, o, gate, up, down}).
    \item \textbf{Full-parameter fine-tuning (Full FT)}: all parameters updated in BF16, while the rollout policy runs with FP8 quantization (E4M3) at inference time.
\end{itemize}
Optimization uses AdamW~\citep{loshchilov2017decoupled} with $\beta_1 = 0.9$, $\beta_2 = 0.99$, weight decay $0.1$, and gradient clipping at $0.2$. Core hyperparameters are summarized in Table~\ref{tab:hyperparams}.

\begin{table}[htbp]
\centering
\caption{Hyperparameter settings for GRPO training across tasks and training modes.}
\label{tab:hyperparams}
\renewcommand{\arraystretch}{1.2}
\resizebox{\columnwidth}{!}{
\begin{tabular}{lccc}
\hline
\textbf{Hyperparameter} & \textbf{Full-parameter FT} & \textbf{LoRA (Qwen3.5-9B)} & \textbf{LoRA (LLaDA-8B)} \\ \hline
Learning rate & $1\!\times\!10^{-6}$ to $1\!\times\!10^{-5}$ & $1\!\times\!10^{-7}$ to $1\!\times\!10^{-6}$ & $1\!\times\!10^{-6}$ to $3\!\times\!10^{-6}$ \\
PPO clipping ($\epsilon$) & $0.2$ & $0.2$ & $0.2$ \\
Number of generations & $8$ & $8$ & $8$ \\
Max prompt length & $256$ & $256$ & $256$ \\
Max completion length & $256 \sim 512$ & $256 \sim 512$ & $256 \sim 1024$ \\
Effective batch size & $96 \sim 128$ & $128$ & $96$ \\ \hline
\end{tabular}
}
\end{table}

\subsection{LLaDA Experiments}
\label{appendix:llada-details}

\paragraph{Model and decoding.}
We fine-tune LLaDA-8B-Instruct~\citep{nie2025large} using semi-autoregressive decoding. For a sequence of length $N$, we perform $N/2$ denoising steps, unmasking the two highest-confidence tokens within the current block at each step. We use block size $32$ and decode left-to-right, which we found to yield better training stability than fully random-order denoising.

\paragraph{Task-specific recipe.}
Training lengths are $7{,}700$ steps for GSM8K and $6{,}600$ for MATH500. Each prompt token is randomly masked with $p_{\text{mask}} = 0.15$ during log-probability estimation, following~\citet{zhao2025d1} to align with the LLaDA masked-prediction objective. We use a composite reward combining (i) a \emph{format reward} for proper XML tagging (\texttt{<reasoning>}, \texttt{<answer>}) and answer delimiters, and (ii) a \emph{correctness reward} (exact match for GSM8K, boxed-answer equivalence for MATH500, and valid arithmetic verification for Countdown).

\paragraph{Deviation from the reference recipe.}
The reference codebase uses $4$-bit quantization during training; we instead keep the learner in BF16, in order to isolate the distributional effects of FP8 \emph{rollout} quantization from any training-side quantization noise.

\subsection{Qwen Experiments}
\label{appendix:qwen-details}

We adapt the GRPO pipeline to Qwen3.5-9B-Instruct. Compared to the LLaDA setup, the main differences are: standard autoregressive decoding (temperature $0.9$) in place of semi-autoregressive denoising, no prompt-masking step, extended sequence length ($1024$ tokens) to accommodate longer chain-of-thought completions, and a cosine learning-rate schedule with $5\%$ linear warm-up. vLLM~\citep{kwon2023efficient} serves FP8 rollouts, with learner and rollout engine on disjoint groups and weight synchronization at each training step.

\subsection{Low-bit Rollout and AIS}
\label{appendix:fp8-ais}

\paragraph{FP8 rollout.}
The rollout policy is quantized to FP8 (E4M3) with per-tensor dynamic scaling, executed via NVIDIA's \texttt{TransformerEngine} for native FP8 tensor-core kernels. Only the rollout forward pass runs in FP8; log-probability evaluation under $\pi_{\mathrm{train}}$ and backpropagation remain in BF16.

\paragraph{AIS hyperparameters.}
Tuned on a GSM8K validation split and held fixed across all LLaDA and Qwen tasks:
\begin{itemize}[leftmargin=2em,itemsep=0.1em]
    \item Clipping threshold $C = 5$: caps the per-token importance weight.
    \item Saturation threshold $\delta = 0.02$: gates correction on the mean absolute log-probability drift.
    \item Variance tolerance $\gamma = 1.2$: suppresses correction when IS variance threatens gradient stability.
    \item Numerical stabilizer $\epsilon = 10^{-6}$.
\end{itemize}
All three signals $\alpha_{\text{ess}}, \alpha_{\text{mis}}, \alpha_{\text{var}}$ are computed from detached tensors to avoid spurious coupling with the policy gradient (cf.~Appendix~\ref{app:theory}).

\subsection{Evaluation Protocol}
\label{appendix:eval-details}

\paragraph{DeepScaleR-style evaluation (AMC).}
For AMC 2022--2023, we follow the DeepScaler~\citep{deepscaler2025} protocol: $n = 16$ samples per problem with temperature $0.6$ and top-$p = 0.95$. We report Pass@1 (mean accuracy over the $16$ samples) and Pass@16 (fraction of problems with at least one correct sample). Verification uses SymPy-based symbolic equivalence.

\paragraph{Deterministic evaluation.}
For GSM8K and MATH500 we report Pass@1 with greedy decoding ($T = 0$). For AIME24 we report Pass@1 averaged over $8$ samples at $T = 0.6$. LLaDA evaluations retain semi-autoregressive decoding; Qwen evaluations use the \texttt{lm-evaluation-harness} with DeepScaler regex-based answer extraction.

\section{Theoretical Analysis of the AIS Estimator}
\label{app:theory}

This appendix provides the formal analysis underpinning the AIS estimator introduced in Section~\ref{sec:method}. We (i) derive the oracle mixing coefficient $\alpha^{\star}$ under a surrogate mean-squared error criterion (Proposition~\ref{prop:oracle}), (ii) establish a bounded second moment for the truncated AIS estimator (Proposition~\ref{prop:variance}), and (iii) demonstrate that AIS exactly recovers the on-policy gradient when rollout-training mismatch is absent (Proposition~\ref{prop:consistency}). Throughout this analysis, expectations $\mathbb{E}_{\mathrm{rollout}}$ and $\mathbb{E}_{\mathrm{train}}$ are taken with respect to $\pi_{\mathrm{rollout}}$ and $\pi_{\mathrm{train}}$ respectively, and $\|\cdot\|$ denotes the Euclidean norm.

\subsection{Setup and Notation}
\label{app:setup}

Let $\pi_{\mathrm{train}}(\cdot)$ and $\pi_{\mathrm{rollout}}(\cdot)$ represent two distributions over trajectories $x$, where $\mathrm{supp}(\pi_{\mathrm{train}}) \subseteq \mathrm{supp}(\pi_{\mathrm{rollout}})$. Let $A(x) \in \mathbb{R}$ be a bounded advantage function and $s(x) := \nabla_\theta \log \pi_{\mathrm{train}}(x) \in \mathbb{R}^{d}$ denote the score vector. The importance ratio and its clipped version are defined as:
\begin{equation}
    w(x) = \frac{\pi_{\mathrm{train}}(x)}{\pi_{\mathrm{rollout}}(x)}, \qquad \bar{w}(x) = \min\bigl(w(x),\, C\bigr), \qquad C \geq 1.
\end{equation}
For a fixed scalar $\alpha \in [0,1]$, we define the mixed AIS estimator:
\begin{equation}
    \hat{g}(\alpha) = \mathbb{E}_{\mathrm{rollout}}\!\left[\bigl(1 - \alpha + \alpha\,\bar{w}(x)\bigr)\, A(x)\, s(x)\right].
    \label{eq:app-ais-estimator}
\end{equation}
We evaluate $\hat{g}(\alpha)$ against the true on-policy gradient $g = \mathbb{E}_{\mathrm{train}}[A(x)\, s(x)]$. We define the following components:
\begin{align}
    \hat{g}_0 &= \mathbb{E}_{\mathrm{rollout}}[A(x)\, s(x)], \qquad b_0 = \hat{g}_0 - g, \\
    \hat{g}_1 &= \mathbb{E}_{\mathrm{rollout}}[\bar{w}(x)\, A(x)\, s(x)], \qquad b_1 = \hat{g}_1 - g,
\end{align}
such that $\hat{g}(\alpha) = (1-\alpha)\,\hat{g}_0 + \alpha\,\hat{g}_1$. Here, the bias $b_0$ originates from the rollout-training mismatch, while $b_1$ represents the residual truncation bias, which vanishes as $C \to \infty$.

\paragraph{Assumptions.} We assume the following regularity conditions:
\begin{itemize}[leftmargin=2em,itemsep=0.2em,topsep=0.2em]
    \item[(A1)] \emph{Bounded advantage:} There exists $M_A < \infty$ such that $|A(x)| \leq M_A$ almost surely.
    \item[(A2)] \emph{Bounded score:} There exists $M_s < \infty$ such that $\|s(x)\| \leq M_s$ almost surely.
    \item[(A3)] \emph{Absolute continuity:} $\mathrm{supp}(\pi_{\mathrm{train}}) \subseteq \mathrm{supp}(\pi_{\mathrm{rollout}})$.
    \item[(A4)] \emph{Negligible truncation bias:} $\|b_1\| \ll \|b_0\|$, allowing us to set $b_1 = 0$ in the surrogate analysis.
\end{itemize}

\subsection{Oracle Mixing Coefficient}
\label{app:oracle}

We consider the surrogate mean-squared error (MSE) as a proxy for estimation quality:
\begin{equation}
    \mathrm{MSE}(\alpha) = \|\mathbb{E}[\hat{g}(\alpha)] - g\|^{2} + \mathrm{tr}\,\mathrm{Cov}[\hat{g}(\alpha)].
    \label{eq:app-mse}
\end{equation}
The following proposition identifies the optimal $\alpha$ that minimizes this criterion.

\begin{proposition}[Oracle mixing coefficient]
\label{prop:oracle}
Under (A1)--(A4), let $\sigma_1^{2} = \mathrm{tr}\,\mathrm{Cov}_{\mathrm{rollout}}[\bar{w}(x)\, A(x)\, s(x)]$ and $\sigma_0^{2} = \mathrm{tr}\,\mathrm{Cov}_{\mathrm{rollout}}[A(x)\, s(x)]$ be the per-sample covariance traces of the two endpoint estimators, and let $\kappa = \mathrm{tr}\,\mathrm{Cov}_{\mathrm{rollout}}[A(x)\,s(x),\, \bar{w}(x)\,A(x)\,s(x)]$ be their cross-covariance trace. The MSE is minimized at
\begin{equation}
    \alpha^{\star}_{\mathrm{exact}} = \frac{\|b_0\|^{2} + \sigma_0^{2} - \kappa}{\|b_0\|^{2} + \sigma_0^{2} + \sigma_1^{2} - 2\kappa}.
    \label{eq:app-oracle-exact}
\end{equation}
In the regime where $\sigma_1^{2}$ dominates $\sigma_0^{2}$ and $\kappa$, equation~\eqref{eq:app-oracle-exact} simplifies to the form used in our practical implementation:
\begin{equation}
    \alpha^{\star} = \frac{\|b_0\|^{2}}{\|b_0\|^{2} + \sigma_1^{2}}.
    \label{eq:app-oracle-alpha}
\end{equation}
\end{proposition}
The simplified form~\eqref{eq:app-oracle-alpha} follows by assuming the regime $\sigma_1^{2} \gg \sigma_0^{2}$ and $\sigma_1^{2} \gg |\kappa|$. This approximation is consistent with empirical findings in importance sampling, where the reweighted variance $\sigma_1^{2}$ typically dominates the total error budget.
\begin{proof}
Substituting $\hat{g}(\alpha) = (1-\alpha)\hat{g}_0 + \alpha \hat{g}_1$ into~\eqref{eq:app-mse} and using $b_1 = 0$ from (A4) gives
\begin{align}
    \mathrm{MSE}(\alpha)
    &= \|(1-\alpha) b_0\|^{2} + \mathrm{tr}\,\mathrm{Cov}\bigl[(1-\alpha)\hat{g}_0 + \alpha \hat{g}_1\bigr] \\
    &= (1-\alpha)^{2}\|b_0\|^{2} + (1-\alpha)^{2}\sigma_0^{2} + \alpha^{2}\sigma_1^{2} + 2\alpha(1-\alpha)\,\kappa.
\end{align}
Differentiating with respect to $\alpha$ and setting to zero yields~\eqref{eq:app-oracle-exact}. The simplified form~\eqref{eq:app-oracle-alpha} follows by setting $\sigma_0^{2} = \kappa = 0$, an approximation supported by the empirical observation that the clipped IS variance $\sigma_1^{2}$ dominates both quantities in the regimes we encounter.
\end{proof}

\subsection{Bounded Second Moment and Consistency}

\begin{proposition}[Bounded second moment]
\label{prop:variance}
Under (A1)--(A3), for any $\alpha \in [0,1]$ and $C \geq 1$, the AIS estimator satisfies
\begin{equation}
    \mathbb{E}_{\mathrm{rollout}}\|\hat{g}(\alpha)\|^{2} \leq \bigl(1 + \alpha(C-1)\bigr)^{2}\, M_A^{2}\, M_s^{2}.
\end{equation}
In particular, $\hat{g}(\alpha)$ has a bounded second moment for any $\alpha \in [0,1]$, ensuring that gradient estimates remain integrable.
\end{proposition}

\begin{proof}
Since $\bar{w}(x) \in [0, C]$, the combined weight satisfies $0 \leq 1 - \alpha + \alpha\,\bar{w}(x) \leq 1 + \alpha(C-1)$. Squaring this bound and applying (A1) and (A2) completes the proof.
\end{proof}

\begin{proposition}[On-policy recovery]
\label{prop:consistency}
If $\pi_{\mathrm{rollout}} = \pi_{\mathrm{train}}$ almost everywhere, then $w(x) = \bar{w}(x) = 1$ almost surely. In this case, for any gating function $\alpha(x) \in [0,1]$, the AIS estimator yields $\hat{g} = g$.
\end{proposition}

\begin{proof}
Under $\pi_{\mathrm{rollout}} = \pi_{\mathrm{train}}$ we have $w(x) = 1$ a.s., hence $\bar{w}(x) = \min(1, C) = 1$ since $C \geq 1$. The mixing factor $1 - \alpha(x) + \alpha(x)\cdot 1 = 1$ for any $\alpha(x) \in [0, 1]$, so
\begin{equation}
    \hat{g} = \mathbb{E}_{\mathrm{rollout}}[A(x)\, s(x)] = \mathbb{E}_{\mathrm{train}}[A(x)\, s(x)] = g. \qedhere
\end{equation}
\end{proof}

\section{Additional Experimental Results}
\label{appendix:additional-results}
\subsection{Training Dynamics and Mechanism Analysis}
\label{sec:dynamics}

\begin{wrapfigure}{r}{0.5\textwidth}
  \centering
  \vspace{-12pt}
  \includegraphics[width=0.38\textwidth]{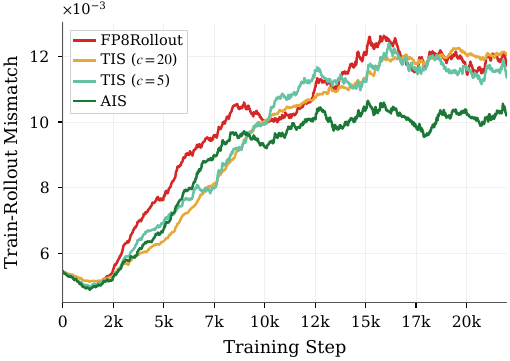}
  \caption{%
    Train--rollout mismatch on GSM8K across four configurations.
  }
  \label{fig:mismatch}
  \vspace{-10pt}
\end{wrapfigure}

To understand why AIS not only recovers but surpasses the BF16 baseline, we analyze the training dynamics through two diagnostic lenses.

\begin{table}[t]
\centering
\small
\setlength{\tabcolsep}{4pt}
\renewcommand{\arraystretch}{1.15}
\caption{LoRA fine-tuning performance on LLaDA-8B across reasoning benchmarks. }
\begin{tabular}{l l c c c c}
\toprule
\textbf{Method} & \textbf{Model}
& \textbf{GSM8K}
& \textbf{MATH500}
& \textbf{Countdown}
& \textbf{Sudoku} \\
\midrule
BF16 (Original) & LLaDA-8B & 74.10 & \textbf{33.70} & 38.20 & 23.89  \\
FP8 Rollout     & LLaDA-8B & 71.42 & 23.20 & 31.64 & 23.73 \\
FlashRL (TIS)   & LLaDA-8B & 72.10 & 33.40 & 35.16 & 24.16 \\
\textbf{AIS (Ours)} & LLaDA-8B & \textbf{74.37} & 33.60 & \textbf{41.41} & \textbf{24.50} \\
\bottomrule
\end{tabular}
\label{tab:main_benchmark_lora_llada}
\end{table}

\paragraph{Mismatch suppression.}
Figure~\ref{fig:mismatch} tracks the per-step train--rollout probability divergence on GSM8K. AIS (dark green) maintains the lowest mismatch throughout training, whereas TIS (light green) tracks the mismatch with increasing delay as training progresses. This is a direct consequence of AIS's adaptive coefficient: as divergence grows, both the necessity signal $\alpha_{\text{mis}}$ and the penalty $\alpha_{\text{var}}$ respond, tightening correction in step with the evolving mismatch. Static truncation, by contrast, has no mechanism to respond to drift.

\paragraph{Stability of reward trajectories.}
Figure~\ref{fig:reward_curves} shows training reward curves on Countdown and MATH. FP8 Rollout (red) exhibits oscillating trajectories characteristic of high-variance gradient estimation, consistent with the uncorrected bias in Eq.~\eqref{eq:oracle-alpha}. AIS and TIS both stabilize training, but AIS achieves higher terminal reward on Countdown with lower variance, indicating that adaptive correction converts the unstable regime of FP8 rollout into a productive learning signal.

\begin{figure}[h]
  \centering
  \includegraphics[width=0.5\linewidth]{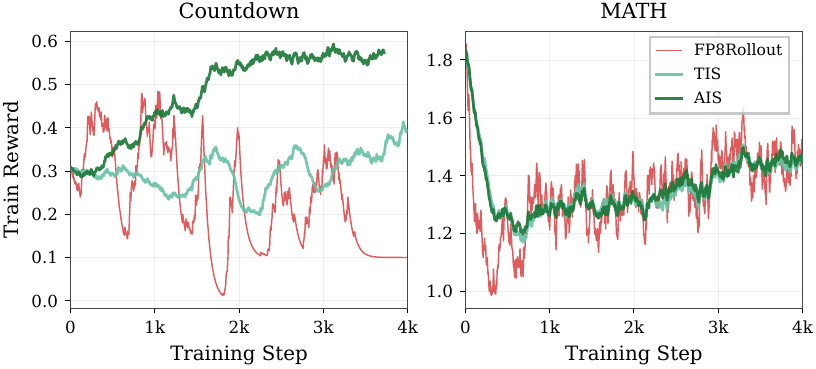}
  \caption{%
    Training reward on Countdown (left) and MATH (right).
    FP8 Rollout (red) exhibits pronounced instability. TIS (light green) and AIS (dark green) both stabilize training.
  }
  \label{fig:reward_curves}
\end{figure}

\paragraph{Quantization noise as implicit exploration.}
A recurring observation is that AIS \emph{exceeds} the BF16 upper bound on several benchmarks: $+6.63\%$ on AIME25 (Qwen3-8B) and $+2.40\%$ on MATH500 (Qwen3.5-9B) in Table~\ref{tab:ar-reasoning}, and $+9.50\%$ on Countdown at $128$ tokens (LLaDA-8B) in Table~\ref{tab:bf16-fp8-tis-ais}. We hypothesize that the stochastic perturbations introduced by FP8 quantization act as an \emph{implicit exploration bonus}, driving the policy into low-probability regions of the trajectory space. When these trajectories are appropriately re-weighted by AIS, the resulting gradient benefits from a broader exploration landscape while remaining approximately unbiased, echoing classical results on noise-aided non-convex optimization~\citep{ge2015escaping}. A formal characterization is left to future work.


\subsection{General-Capability Benchmarks for Autoregressive Models}
\label{appendix:ar-general}

Table~\ref{tab:ar-general} reports zero-shot accuracy on six general-capability benchmarks, serving as a sanity check that RL fine-tuning under reduced rollout precision does not cause regressions outside the target tasks. Naive FP8 rollout noticeably degrades accuracy on both Qwen3-8B and Qwen3.5-9B (e.g., ARC-C drops from $54.44$ to $51.52$, WinoGrande from $73.09$ to $70.69$ on Qwen3.5-9B). TIS partially recovers the lost accuracy but does not fully close the gap. AIS matches or exceeds the BF16 baseline on nearly all benchmarks while operating entirely under FP8 rollout, with the only deficit (WinoGrande on Qwen3.5-9B, $72.93$ vs.\ $73.09$) within evaluation noise.

\begin{table}[t]
\centering
\small
\caption{Evaluation accuracy (\%) on general-capability benchmarks for autoregressive models.}
\resizebox{\textwidth}{!}{%
\begin{tabular}{ll c cccccc}
\toprule
\textbf{Model} & \textbf{Method} & \textbf{Bitwidth} & \textbf{MMLU} & \textbf{PIQA} & \textbf{ARC-E} & \textbf{ARC-C} & \textbf{HellaSwag} & \textbf{WinoGrande} \\
\midrule
\multirow{4}{*}{Qwen3-8B}
  & RL (BF16)        & BF16 & 73.02 & 76.50 & 83.33 & 55.46 & 74.89 & 67.96 \\
\cmidrule(lr){2-9}
  & RL (FP8 Rollout) & FP8  & 72.92 & 76.82 & 83.59 & 55.80 & 74.95 & 67.88 \\
  & FlashRL (TIS)    & FP8  & 73.00 & 77.31 & 82.74 & 55.72 & 77.12 & 71.43 \\
  & AIS (Ours)       & FP8  & \textbf{73.34} & \textbf{77.75} & \textbf{83.54} & \textbf{55.97} & \textbf{77.44} & \textbf{73.48} \\
\midrule
\multirow{4}{*}{Qwen3.5-9B}
  & RL (BF16)        & BF16 & 78.61 & 79.27 & 81.14 & 54.44 & 78.06 & 73.09 \\
\cmidrule(lr){2-9}
  & RL (FP8 Rollout) & FP8  & 77.61 & 78.00 & 80.44 & 51.52 & 75.02 & 70.69 \\
  & FlashRL (TIS)    & FP8  & 77.69 & 78.78 & 81.73 & 54.52 & 77.12 & 71.24 \\
  & AIS (Ours)       & FP8  & \textbf{78.65} & \textbf{79.22} & \textbf{81.57} & \textbf{54.61} & \textbf{78.16} & 72.93 \\
\bottomrule
\end{tabular}%
}
\label{tab:ar-general}
\end{table}

\subsection{BF16 vs.\ FP8 Rollout Benchmark on LLaDA-8B-Instruct}
\label{appendix:bf16-fp8-benchmark}

\begin{table*}[t]
\centering
\scriptsize
\begin{threeparttable}
\caption{BF16 vs.\ FP8 Rollout benchmark on LLaDA-8B-Instruct (max length = 256). Results averaged over per-step timings; speedup computed as BF16/FP8.}
\label{tab:bf16-fp8-benchmark}
\begin{tabular}{ccccccccccc}
\toprule
\textbf{BS} & \textbf{G} &
\textbf{Generate (s)} & \textbf{Spd} &
\textbf{Rollout (s)} & \textbf{Spd} &
\textbf{Ref (s)} & \textbf{Spd} &
\textbf{Total (s)} & \textbf{Spd} &
\textbf{Peak Mem (GB)} \\
\midrule
4  & 4  & 46.96$\rightarrow$16.57 & 2.83$\times$ & 0.20$\rightarrow$0.19 & 1.07$\times$ & 0.20$\rightarrow$0.24 & 0.82$\times$ & 47.56$\rightarrow$17.23 & 2.76$\times$ & 50.94$\rightarrow$53.98 \\
8  & 4  & 39.03$\rightarrow$16.13 & 2.42$\times$ & 0.41$\rightarrow$0.36 & 1.12$\times$ & 0.43$\rightarrow$0.37 & 1.14$\times$ & 40.27$\rightarrow$17.04 & 2.36$\times$ & 61.07$\rightarrow$62.41 \\
8  & 8  & 28.20$\rightarrow$17.76 & 1.59$\times$ & 0.40$\rightarrow$0.36 & 1.11$\times$ & 0.42$\rightarrow$0.38 & 1.12$\times$ & 29.41$\rightarrow$18.76 & 1.57$\times$ & 60.63$\rightarrow$62.02 \\
16 & 4  & 50.54$\rightarrow$43.41 & 1.16$\times$ & 0.86$\rightarrow$0.75 & 1.15$\times$ & 0.86$\rightarrow$0.85 & 1.02$\times$ & 53.11$\rightarrow$45.96 & 1.16$\times$ & 62.85$\rightarrow$63.88 \\
16 & 8  & 47.72$\rightarrow$34.87 & 1.37$\times$ & 0.84$\rightarrow$0.71 & 1.18$\times$ & 0.85$\rightarrow$1.02 & 0.83$\times$ & 50.23$\rightarrow$37.17 & 1.35$\times$ & 61.07$\rightarrow$62.41 \\
\bottomrule
\end{tabular}

\end{threeparttable}
\end{table*}

Table~\ref{tab:bf16-fp8-benchmark} reports wall-clock and memory measurements on LLaDA-8B-Instruct at a fixed maximum length of $256$ tokens, decomposing per-step latency into \emph{Generate} (denoising), \emph{Rollout} (sampling-side bookkeeping), and \emph{Ref} (reference log-probability recomputation). Three observations stand out. \emph{(i)} Generation accounts for over $95\%$ of total step time, so rollout-side acceleration directly translates into end-to-end speedup. \emph{(ii)} FP8 delivers the largest gains at small batch sizes ($2.76\times$ end-to-end at BS $= 4$), where decoding is memory-bandwidth bound, and the gains shrink as compute saturates at larger batches ($1.16$--$1.37\times$ at BS $= 16$). \emph{(iii)} FP8 enables configurations infeasible under BF16 (e.g., BS $= 16$ with G $= 8$ runs under FP8 but OOMs under BF16), while adding only $1$--$3$\ GB of peak memory since the learner, optimizer state, and gradients remain in BF16. Combined with the accuracy results in Table~\ref{tab:ar-general}, these measurements confirm that FP8 rollout delivers substantial acceleration at near-zero memory cost, with the precision penalty fully absorbed by AIS.

\section{Limitations}
\label{sec:limitations}

We acknowledge several limitations. First, our experiments focus on $8$--$9$B-parameter models in the BF16-trainer/FP8-rollout setup, and scaling to larger models or more aggressive quantization (e.g., FP4, INT4) remains an open question that may require re-tuning $C$, $\delta$, $\gamma$. Second, we evaluate on mathematical reasoning and planning, leaving instruction-following, dialogue, code generation, and RLHF-style preference learning unexplored. Third, the oracle coefficient in Proposition~\ref{prop:oracle} relies on simplifying assumptions (negligible truncation bias and dominant clipped IS variance) that may not hold under heavier quantization or smaller batches. Fourth, our exploration-bonus claim is supported empirically rather than by a formal causal analysis, and isolating exploration from other mechanisms (gradient regularization, implicit smoothing) is left to future work. Finally, compute constraints limit each configuration to a single seed, and a study of run-to-run variance would further strengthen the empirical claims.